\documentclass[10pt,twocolumn,letterpaper]{article}

\usepackage[letterpaper,margin=0.75in,columnsep=0.25in]{geometry}
\usepackage{times}
\usepackage[round,authoryear]{natbib}
\usepackage{amsmath}
\usepackage{amssymb}
\usepackage{amsthm}
\usepackage[table]{xcolor}
\usepackage{graphicx}
\usepackage{subcaption}
\usepackage{booktabs}
\usepackage{multirow}
\usepackage{float}
\usepackage{pifont}
\usepackage{url}
\usepackage{hyperref}
\hypersetup{colorlinks=true,linkcolor=blue!60!black,citecolor=blue!60!black,urlcolor=blue!60!black}

\newcommand{\E}{\mathbb{E}}
\newcommand{\Cov}{\mathrm{Cov}}
\newcommand{\KL}{D_{\mathrm{KL}}}

\newcommand{\cmark}{\ding{51}}
\newcommand{\xmark}{\ding{55}}

\newtheorem{proposition}{Proposition}[section]
\newtheorem{lemma}[proposition]{Lemma}
\newtheorem{corollary}[proposition]{Corollary}
\theoremstyle{remark}
\newtheorem{remark}[proposition]{Remark}

\newcommand{\tr}{\operatorname{tr}}

\newcommand{\logit}{\operatorname{logit}}
\newcommand{\KLBer}{\operatorname{KL}_{\mathrm{Ber}}}
\newcommand{\pS}{p_{S}}
\newcommand{\pbar}{\bar p}
\newcommand{\phat}{\widehat p}
\newcommand{\ghat}{\widehat g}
\newcommand{\Shat}{\widehat S}
\newcommand{\Vhat}{\widehat V}
\newcommand{\Bhat}{\widehat B}
\newcommand{\what}{\widehat w}
\newcommand{\simplex}{\Delta_{\mathcal V}}

\title{\Large\bfseries MASS: Multi-mask Selective Self-Distillation for Leakage-Resilient Reasoning}

\author{%
Xiuqi Liu\textsuperscript{1}\qquad
Oveys Delafrooz Noroozi\textsuperscript{2}\qquad
Upamanyu Madhow\textsuperscript{2} \\[0.3em]
\textsuperscript{1}Chongqing Bashu Secondary School\qquad
\textsuperscript{2}University of California, Santa Barbara \\[0.2em]
\texttt{lxq20080805@gmail.com},\quad \texttt{\{oveys, madhow\}@ucsb.edu}
}
\date{}

\begin{document}

\maketitle

\begin{abstract}
On-policy self-distillation enables the same model to act as student and teacher, with dense token-level supervision obtained by the teacher being given access to privileged information. 
However, conditioning the teacher on a single reference can concentrate its predictions on one reasoning trajectory and limit transfer to the student. We propose Multi-Mask Selective Self-Distillation (MASS), which approximates the unavailable marginal over diverse reference realizations using multiple independently masked views of one verified solution. MASS aggregates these complementary teacher predictions, selectively transfers reliable token-level corrections, and adaptively focuses masking on influential reference spans. Across mathematical reasoning benchmarks and model families, MASS consistently outperforms strong self-distillation baselines.
\end{abstract}

\section{Introduction}
Large language models (LLMs) demonstrate strong reasoning capabilities through post training approaches such as Reinforcement Learning with Verifiable Rewards \citep{shao2024deepseekmath, guo2025deepseek,ouyang2022training} (RLVR) and Supervised Fine Tuning (SFT) \citep{zelikman2022star}. However, the substantial computational and deployment costs of large models have motivated growing interest in transferring their capabilities to smaller, more efficient models through knowledge distillation \citep{wang2020minilm, hsieh2023distilling, hinton2015distilling}. Conventional off-policy distillation, while effective, suffers from distribution mismatch between teacher demonstrations and student trajectories\citep{ross2011reduction}. On-policy distillation (OPD) \citep{agarwal2024policy} addresses this mismatch by training the student on its own rollouts while obtaining dense token-level supervision from a teacher. On-policy self-distillation (OPSD) \citep{zhao2026self} further removes the need for a separate large teacher: the same model plays both roles, with the teacher distinguished by access to privileged information that is unavailable to the student at inference time. This framework has since been extended to privileged contexts such as rubrics, visual inputs, and long-context information \citep{ye2026policy,zhang2026opsdl,bablani2026rubrics}.

Recent studies have identified several consequences of conditioning a self-teacher on privileged information, including performance collapse, unstable optimization, trajectory-specific preferences, and shifts toward overly succinct reasoning behavior \citep{yang2026self,harne2026privileged,pan2026rlcsd,kim2026does}. Although these effects arise through different mechanisms, they expose a central risk: the teacher may become overly confident in the particular reference it is given, concentrating its supervision on one reasoning trajectory while underrepresenting other continuations that remain plausible from the student's information state.

To address the problem, existing approaches improve the transferability of privileged supervision by anchoring the update with verifier feedback, contrasting different privileged contexts, or reshaping and reweighting the teacher signal \citep{yang2026self,pan2026rlcsd,li2026demopsd,xu2026betaopsdderivingpolicyoptimization,shen2026anti}. However, these methods primarily intervene only after the privileged distribution has been produced by teacher. Since reference-specific bias originates from conditioning the teacher on a complete solution, such post-hoc adjustments may not fully address its source. This motivates intervening at the level of the teacher's conditioning context, thereby shaping the supervision distribution before it is transferred to the student.  A potential solution, suggested by \citet{yang2026self}, is to average teacher predictions across multiple valid solutions, thereby reducing the influence of any single reasoning trajectory. In practice, however, obtaining multiple verified solution for each question is very expensive. 

\noindent
{\bf Our approach:} We generate diverse conditioning contexts for teacher from a single reference by independently masking different spans of the reference, so that their averaged prediction retains useful guidance without being dominated by one complete trajectory. Our proposed method,
\textbf{Multi-Mask Selective Self-Distillation (MASS),} averages teacher predictions from multiple partially masked references to construct a less reference-specific target, selectively weights reliable token-level supervision, and adapts masking toward influential reference spans. Together, these components preserve useful privileged guidance while encouraging support for a broader set of reasoning paths.

\begin{figure}[!t]
\centering
\includegraphics[width=\linewidth]{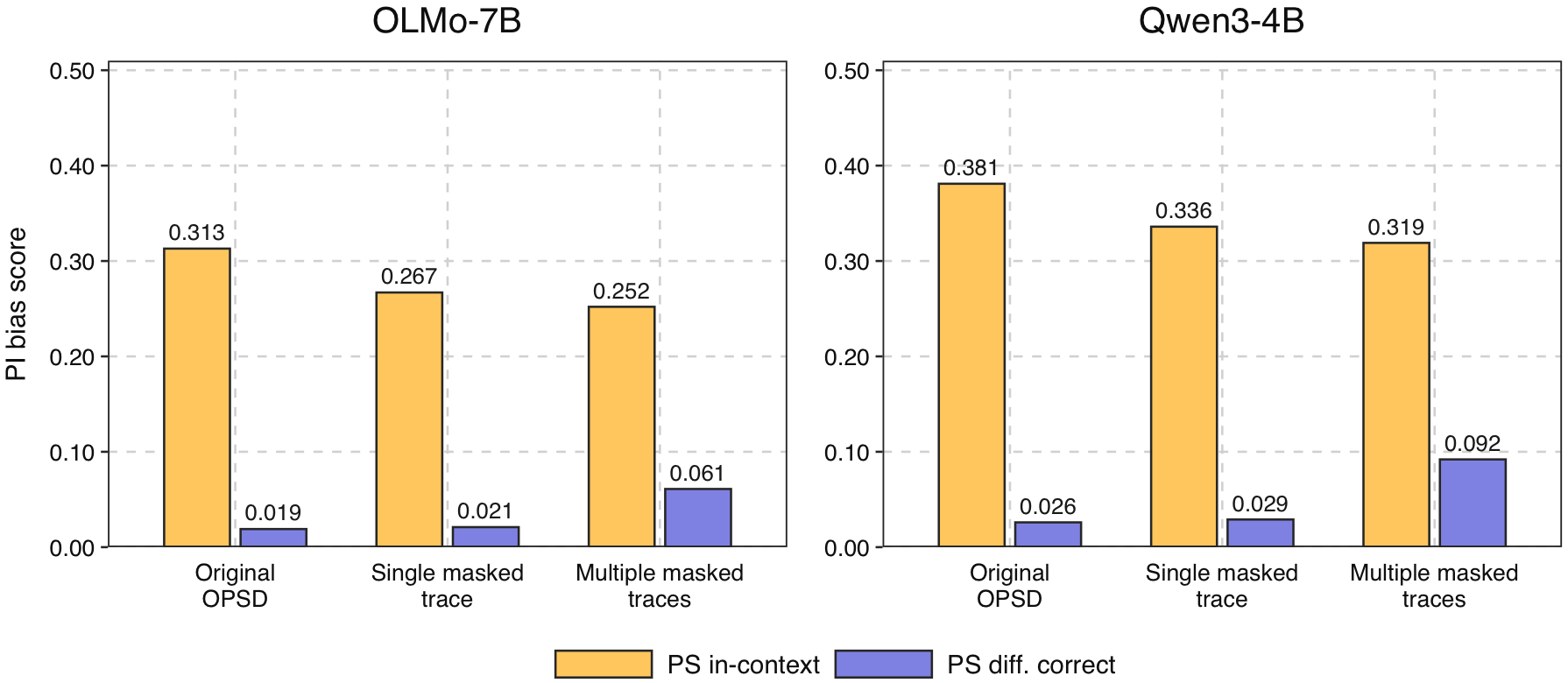}
\caption{PI Bias Scores for the in-context solution and a distinct correct solution under the full-reference teacher, a single masked view, and the arithmetic mean of multiple masked views.}
\label{fig}
\end{figure}

\noindent
{\bf A sanity check:} To test whether masked views recover support for alternative reasoning paths, we estimated the PI Bias Score of \citet{harne2026privileged}. At each student-generated prefix, we ask the privileged teacher and unprivileged student how much probability they assign to the same fixed continuation, and compute the difference in their mean token log-probabilities. We evaluate this continuation either as the reference solution shown to the teacher (\emph{PS in-context}) or as a different verified-correct solution to the same problem (\emph{PS different-correct}). A larger gap between these two scores means the teacher is more anchored to the solution that teacher sees as privileged information, while smaller gap means the teacher is more broadly distributed across valid reasoning alternatives. As shown in Figure~\ref{fig}, a single random mask changes this gap only slightly, whereas averaging multiple independently masked views substantially increases support for the different correct solution while retaining positive support for the in-context one. This suggests that multi-mask averaging can potentially encourage more \emph{diverse} reasoning supervision while still preserving useful reference-based guidance. This is supported by the extensive evaluations of MASS in this paper.

\section{Methodology}
\label{sec:method}

We describe here the key steps in MASS, with detailed derivations provided in Appendix ~\ref{app:theory}.

\subsection{Multi-Mask Teacher}
\label{subsec:mask-formulation}

We build on the on-policy self-distillation (OPSD) reasoning framework. Let $x$ be an input question and $R$ its reference solution, which is available only during training. The student generates an on-policy trajectory $y=(y_1,\ldots,y_T)$ from the question alone. At position $t$, its visible state is $s_t=(x,y_{<t})$, and its next-token distribution is $p_{S,t}=p_S(\cdot\mid s_t)$. The full-reference teacher (anchor teacher) uses the same model but additionally observes $R$, producing $p_{0,t}=p_T(\cdot\mid s_t,R)$. All teacher distributions are detached during the student update.

Unlike prior work, which typically constructs masked views by selecting a cutoff and hiding the remaining suffix \citep{tan2026paint, feng2024keypoint, nguyen2026avsd}, we independently mask spans throughout the reference to produce complementary teacher views. Specifically, we divide $R$ into $J$ spans (parsed based on punctuation) and use a binary vector $m\in\{0,1\}^{J}$ to represent a masked reference $g_m(R)$, where $m_j=1$ indicates that span $j$ is hidden. For each span, the masking decision is sampled independently as $m_j\sim\operatorname{Bernoulli}(q_j)$. The probabilities $q=(q_1,\ldots,q_J)$ are initially set to the same fixed masking rate and are later adapted as described in Section~\ref{subsec:adaptive-mask}. For each student trajectory, we sample $K\geq2$ masks $m^{(1)},\ldots,m^{(K)}$ and evaluate their teacher predictions at the same student prefixes:
\[
p_{k,t}
=
p_T\!\left(
\cdot\mid s_t,g_{m^{(k)}}(R)
\right),
\qquad k=1,\ldots,K.
\]

The training target is the arithmetic mean of the $K$ masked teachers
\begin{equation}
\widehat p_t
=
\frac{1}{K}\sum_{k=1}^{K}p_{k,t}
\label{eq:empirical-marginal}
\end{equation}
which is a stochastic approximation to the population target

\begin{equation}
\bar p_t
=
\mathbb E_{M\sim q}[p_{M,t}],
\label{eq:population-marginal}
\end{equation}
Section~\ref{subsec:selective-distillation} examines the impact of the uncertainty due to finite $K$ and our proposed solution.

\begin{figure*}[t]
\centering
\includegraphics[width=\textwidth]{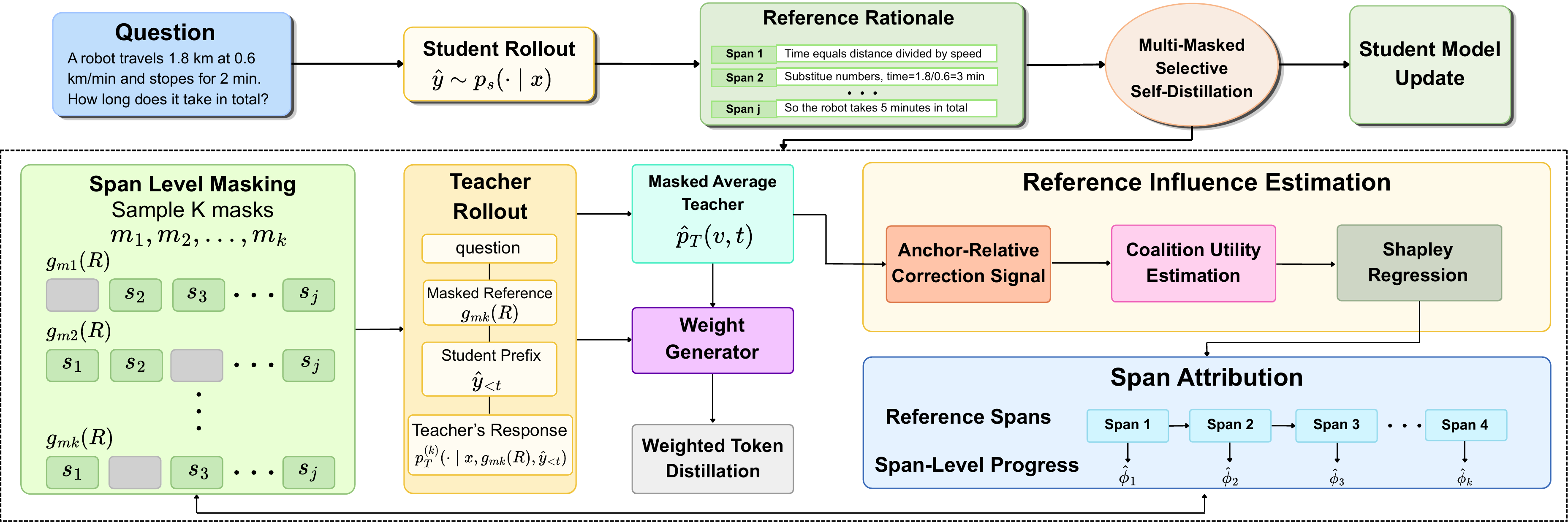}
\caption{\textbf{Overview of MASS.} Multiple masked reference provide complementary teacher predictions over the same student trajectory and are averaged into a marginal target. MASS selectively distills this target using reliability-aware weighting. In parallel, anchor-relative coalition utilities are decomposed  into span-level attributions, which adapt future masking probabilities.}
\label{fig:overview}
\end{figure*}
\subsection{Stable and Reference-Aware Distillation}
\label{subsec:selective-distillation}

Mask averaging does not by itself determine how strongly the student should follow the resulting target. The finite-sample uncertainty in the estimate $\widehat p_t$ should be distinguished from view diversity: different masked teachers may reveal useful alternative continuations, while their empirical average may still be too unstable to provide reliable supervision at a given position. We therefore design the token-level weight from two complementary perspectives. First, it should compare the magnitude of the mask-marginal correction with its sampling uncertainty, retaining corrections with a strong signal relative to noise. Second, it should preserve corrections that move the student away from a preference shared with the full-reference teacher, even when the masked views exhibit moderate disagreement. We use $p_{0,t}$, the teacher distribution conditioned on the complete reference $R$, as an anchor for identifying such shared preferences.

Under forward-KL distillation, the gradient with respect to the student logits is the student distribution minus the detached teacher target. We define the ideal mask-marginal gradient as $g_t=p_{S,t}-\bar p_t$ and its finite-view estimate as $\widehat g_t=p_{S,t}-\widehat p_T$. Independent mask sampling makes  $\widehat g_t$ unbiased, while its covariance is governed by the dispersion of masked-teacher predictions and decreases with $K$. These properties motivate three population quantities: the correction signal $S_t$, its finite-mask uncertainty $V_t$, and its alignment with the full-reference anchor $B_t$:
\begin{equation}
\begin{gathered}
S_t=\lVert g_t\rVert_2^2,
\qquad
V_t=\frac{1}{K}\operatorname{tr}\!\left(
\operatorname{Cov}_{M\sim q}(p_{M,t})
\right),
\\
B_t=
\left\langle
p_{S,t}-\bar p_t,\,
p_{0,t}-\bar p_t
\right\rangle .
\end{gathered}
\label{eq:signal-noise}
\end{equation}
Here, $S_t$ measures the magnitude of the population correction, while
$V_t=\mathbb{E}\lVert\widehat g_t-g_t\rVert_2^2$ measures the expected error caused by estimating it with only $K$ masks. The anchor term $B_t$ is positive when the student and the full-reference teacher deviate from $\bar p_t$ in similar directions. In this case, moving the student toward the mask marginal counteracts a preference reinforced by the complete reference. The three quantities characterize the directional alignment, magnitude, and reliability of the correction, respectively, providing a joint basis for understanding the effect of distillation.

The population quantities in Eq.~\eqref{eq:signal-noise} cannot be computed directly. However, the same $K$ teacher views used to form $\widehat p_t$ provide bias-corrected estimates:
\begin{equation}
\begin{aligned}
\widehat V_t
&=
\frac{1}{K(K-1)}
\sum_{k=1}^{K}
\lVert p_{k,t}-\widehat p_t\rVert_2^2,\\
\widehat S_t
&=
\lVert p_{S,t}-\widehat p_t\rVert_2^2-\widehat V_t,
\\
\widehat B_t
&=
\left\langle
p_{S,t}-\widehat p_t,\,
p_{0,t}-\widehat p_t
\right\rangle-\widehat V_t .
\end{aligned}
\label{eq:corrected-moments}
\end{equation}
The first line estimates the variance of the empirical mean. Because the same sampled mean appears in both quadratic statistics, their uncorrected values contain an additional $\widehat V_t$ term. Subtracting it makes $\widehat V_t$, $\widehat S_t$, and $\widehat B_t$ conditionally unbiased for $V_t$, $S_t$, and $B_t$, respectively.

We next determine how much of the empirical gradient should be retained. For a scalar $w\in[0,1]$, consider the population criterion
$\mathbb{E}\lVert w\widehat g_t-g_t\rVert_2^2
+\varepsilon w^2-2\lambda w[B_t]_+$,
where $[a]_+=\max(a,0)$, $\lambda\geq0$ controls the contribution of the anchor-relative signal, and $\varepsilon>0$ provides numerical regularization. Since
$\mathbb{E}\lVert w\widehat g_t-g_t\rVert_2^2
=(1-w)^2S_t+w^2V_t$,
its clipped minimizer is
$w_t^*=\min\{1,(S_t+\lambda[B_t]_+)/
(S_t+V_t+\varepsilon)\}$.
Replacing the population quantities with their estimates gives the practical weight
\begin{equation}
\widehat w_t=
\min\!\left\{
1,\,
\frac{
[\widehat S_t]_+
+\lambda[\widehat B_t]_+
}{
[\widehat S_t]_+
+\widehat V_t+\varepsilon
}
\right\}.
\label{eq:practical-weight}
\end{equation}
We restrict $\widehat w_t$ to $[0,1]$ because it represents the fraction of the empirical correction retained, rather than an amplification factor. A value near one applies the multi-mask correction almost fully, indicating that it is strong, reproducible, or supported by the anchor-relative signal; a value near zero largely removes a correction that can be explained by mask-sampling uncertainty. The unconstrained ratio may exceed one when the anchor term is sufficiently large, but this would extrapolate beyond the estimated correction and potentially destabilize training. We therefore clip such values to one and interpret them as sufficient evidence to retain the correction in full.
Finally, MASS applies the weight to token-level distillation:
\begin{equation}
\mathcal L(\theta)
=
\frac{1}{T}
\sum_{t=1}^{T}
\operatorname{sg}(\widehat w_t)\,
\operatorname{KL}\!\left(
\operatorname{sg}(\widehat p_t)
\,\Vert\,
p_{S,t}
\right),
\label{eq:weighted-loss}
\end{equation}
where $\operatorname{sg}$ denotes stop-gradient. The scalar weight controls the strength of supervision at each position while preserving the relative probabilities of the arithmetic teacher target.

\subsection{Adaptive Masking}
\label{subsec:adaptive-mask}

Uniform masking assigns the same expected budget to all reference spans, although their effects on the teacher may differ substantially. Some spans provide only local details, while others determine the reasoning direction or final answer. We therefore attribute useful mask-induced changes to individual spans and adapt their masking probabilities on later visits to the same example. The utility of a span cannot be evaluated by masking it in isolation, because the effect depends on which other spans remain visible.  Instead, as discussed below, we treat spans as players in a cooperative game and masked spans as a coalition to compute approximate Shapley values\citep{shapley1953value} for each span.  These are then used to update the masking probabilities.

Let $A\subseteq\{1,\ldots,J\}$ denote a coalition of masked spans and let
$p_{A,t}=p_T(\cdot\mid s_t,g_A(R))$ be the corresponding teacher distribution. During one visit, all coalitions are evaluated against the same detached correction direction:
\begin{equation}
v(A)=\frac{1}{T}\sum_{t=1}^{T}
\left\langle
\operatorname{sg}(p_{S,t}-\widehat p_t),\,
p_{0,t}-p_{A,t}
\right\rangle.
\label{eq:coalition-value}
\end{equation}
A positive $v(A)$ means that masking coalition $A$ changes the teacher in a direction consistent with the student update toward the common mask average. Fixing the first factor makes coalition values comparable. The unmasked teacher provides a zero-utility baseline: when $A=\varnothing$,$p_{\varnothing,t}=p_{0,t}$ and $v(\varnothing)=0$.

For a coalition $A$ that excludes span $j$, define the marginal contribution of $j$ as
$\delta_j(A)=v(A\cup\{j\})-v(A)$. The Shapley value of span $j$ is
\begin{equation}
\phi_j=
\sum_{A\subseteq\{1,\ldots,J\}\setminus\{j\}}
\frac{|A|!(J-|A|-1)!}{J!}\,
\delta_j(A).
\label{eq:shapley}
\end{equation}
Hence, $\phi_j>0$ indicates that masking span $j$ supports the current mask-marginal correction on average over the states of the other spans.

Instead of considering the $2^J$ coalitions required for the exact computation in (\ref{eq:shapley}), we reuse the $K$ masks already evaluated for distillation. For mask $m^{(k)}\in\{0,1\}^J$, let $A_k$ denote its masked-span coalition and $R^{(k)}=v(A_k)$ its utility. We estimate span attributions using a Shapley-kernel weighted regression \citep{lundberg2017unified}:
\begin{equation}
\begin{split}
\widehat\phi={}&
\arg\min_{\beta}
\sum_{k=1}^{K}\omega_k
\left(R^{(k)}-m^{(k)\top}\beta\right)^2
+\tau\lVert\beta-z\rVert_2^2
\\
&\text{s.t.}
\quad
\sum_{j=1}^{J}\beta_j=v(\{1,\ldots,J\}).
\end{split}
\label{eq:shapley-regression}
\end{equation}
where $\beta_j$ is the attribution assigned to span $j$, $z_j$ is its running estimate from previous visits updated through exponential moving average, and $\tau\geq0$ stabilizes the fit when $K<J$. The weight $\omega_k$ is the Shapley-kernel weight corrected by the sampling probability of $m^{(k)}$ under $q$.

We set the masking priority of span $j$ to
$I_j=[\widehat\phi_j]_+$.
Let $q_j$ be the probability of masking span $j$, so
$\operatorname{Ber}(q_j)$ denotes the corresponding Bernoulli mask distribution. We update $q$ by
\begin{equation}
\begin{split}
q^{(n+1)}
={}&
\arg\max_{q\in\mathcal Q}
\left\{
\sum_{j=1}^{J}I_jH_{\mathrm{Ber}}(q_j)
\vphantom{\frac{1}{\eta}}
\right.
\\
&\left.
-\frac{1}{\eta}\sum_{j=1}^{J}
\operatorname{KL}\!\left(
\operatorname{Ber}(q_j)
\,\Vert\,
\operatorname{Ber}(q_j^{(n)})
\right)
\right\},
\end{split}
\label{eq:mask-update}
\end{equation}
where
$H_{\mathrm{Ber}}(q_j)=-q_j\log q_j-(1-q_j)\log(1-q_j)$,
$\mathcal Q=\{q:\sum_jq_j=\rho J,\,
\epsilon\leq q_j\leq1-\epsilon\}$,
and $\eta>0$ controls the adaptation rate.
The entropy term keeps influential spans sampled in both masked and visible states, while the KL term prevents abrupt changes in their masking probabilities.

\section{Experiments}

\begin{table}[!t]
\centering
\setlength{\tabcolsep}{3.5pt}
\small
\caption{
\textbf{Performance comparison on mathematical reasoning benchmarks.}
We report Avg@12 accuracy on AIME 2024, AIME 2025, and HMMT 2025.
MASS consistently achieves the strongest average performance across all evaluated model settings.
}
\label{tab:main_results}

\begin{tabular}{llcccc}
\toprule
Model & Method & AIME24 & AIME25 & HMMT25 & Avg. \\
\midrule

\multirow{5}{*}{Qwen3-1.7B}
& Base & 51.39 & 36.67 & 23.06 & 37.04 \\
& OPSD & 55.56 & 41.11 & 26.94 & 41.20 \\
& AVSD & 59.17 & 45.28 & 29.72 & 44.72 \\
& RLSD & 58.89 & 45.83 & 30.00 & 44.91 \\
& \cellcolor{gray!15}\textbf{MASS}
& \cellcolor{gray!15}\textbf{60.83}
& \cellcolor{gray!15}\textbf{47.50}
& \cellcolor{gray!15}\textbf{30.83}
& \cellcolor{gray!15}\textbf{46.39} \\
\midrule

\multirow{5}{*}{Qwen3-4B}
& Base & 75.00 & 66.39 & 42.22 & 61.20 \\
& OPSD & 76.39 & 68.33 & 43.33 & 62.68 \\
& AVSD & 78.33 & 72.78 & 45.56 & 65.56 \\
& RLSD & 78.06 & 72.50 & 45.28 & 65.28 \\
& \cellcolor{gray!15}\textbf{MASS}
& \cellcolor{gray!15}\textbf{81.11}
& \cellcolor{gray!15}\textbf{73.89}
& \cellcolor{gray!15}\textbf{48.06}
& \cellcolor{gray!15}\textbf{67.69} \\
\midrule

\multirow{5}{*}{Qwen3-8B}
& Base & 78.89 & 70.83 & 48.33 & 66.02 \\
& OPSD & 81.67 & 74.44 & 50.83 & 68.98 \\
& AVSD & 83.89 & 77.50 & 53.06 & 71.48 \\
& RLSD & 83.61 & 77.22 & 52.78 & 71.20 \\
& \cellcolor{gray!15}\textbf{MASS}
& \cellcolor{gray!15}\textbf{85.56}
& \cellcolor{gray!15}\textbf{79.44}
& \cellcolor{gray!15}\textbf{54.72}
& \cellcolor{gray!15}\textbf{73.24} \\
\midrule

\multirow{5}{*}{OLMo-7B}
& Base & 71.9 & 66.9 & 45.0 & 61.3 \\
& OPSD & 73.1 & 70.0 & 46.1 & 63.1 \\
& AVSD & 74.7 & 71.9 & 47.8 & 64.8 \\
& RLSD & 75.6 & 72.8 & 48.3 & 65.6 \\
& \cellcolor{gray!15}\textbf{MASS}
& \cellcolor{gray!15}\textbf{78.1}
& \cellcolor{gray!15}\textbf{75.0}
& \cellcolor{gray!15}\textbf{51.1}
& \cellcolor{gray!15}\textbf{68.1} \\

\bottomrule
\end{tabular}
\end{table}

\subsection{Experimental Setup}
\label{subsec::eval}
We conduct experiments on Qwen3-1.7B, Qwen3-4B, Qwen3-8B
\citep{yang2025qwen3}, and OLMo-7B\citep{olmo2026olmo3}, covering multiple model scales and
families to evaluate whether MASS generalizes across both weaker and stronger
reasoners as well as different model families. Following the OPSD training
protocol, we use 30K mathematical reasoning samples from OpenThoughts
\citep{guha2025openthoughts} for training.

We evaluate on AIME 2024, AIME 2025, and HMMT 2025
\citep{dekoninck2026matharena}. For each problem, we sample 12 independent
reasoning trajectories, resulting in 360 responses per benchmark. We report
Avg@12 accuracy, defined as the proportion of sampled responses that reach the
correct final answer.

Experiment details for the MASS implementation are provided Appendix C, while additional experimental results are provided in Appendix B.

\begin{figure*}[t]
\centering
\includegraphics[width=0.95\textwidth]{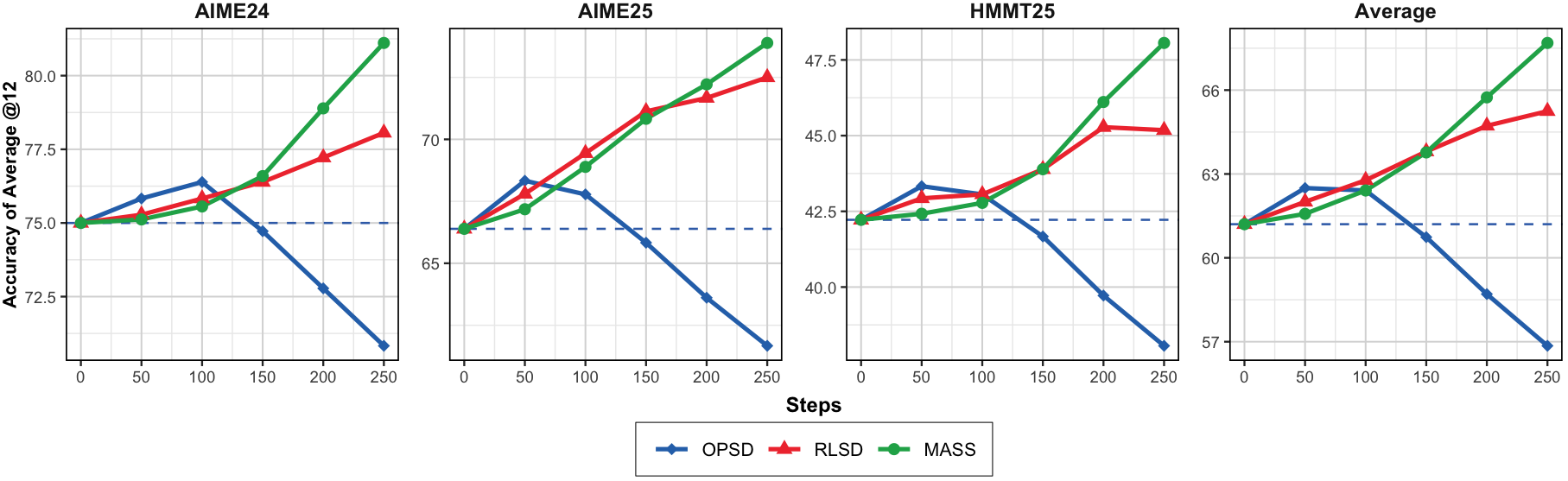}
\caption{\textbf{Performance Dynamics over 250 Training Steps.} We compare MASS
with OPSD and RLSD on Qwen3-4B reporting Avg@12 accuracy on AIME25 within 250 training
steps. OPSD peaks early, while RLSD and MASS sustain improvement throughout
training.}
\label{fig:training_dynamics}
\end{figure*}

We compare MASS with three representative post-training baselines. OPSD
\citep{zhao2026self} is the direct single-view
self-distillation baseline, where the teacher is conditioned on the full
reference solution. RLSD \citep{yang2026self} addresses
privileged-information leakage through an RLVR-based formulation. Although it
belongs to a different methodological family, we include it as a strong
reasoning baseline and to compare token efficiency, since RLSD relies on
grouped rollouts and verifiable rewards while MASS performs dense token-level
distillation from a single student trajectory. AVSD \citep{nguyen2026avsd} is the closest multi-view baseline. It combines
heterogeneous privileged views by separating cross-view consensus from
view-specific residual signals. In contrast, MASS constructs complementary
views by masking different spans of the same reference and uses arithmetic
aggregation to preserve alternative reasoning branches.

\subsection{Main Results}
Table~\ref{tab:main_results} shows that MASS consistently outperforms the compared post-training baselines across model scales and model families. The improvement over OPSD indicates that directly distilling from a fully privileged teacher can become limiting when the supervision is overly tied to a particular reference trajectory. MASS also remains stronger than AVSD, suggesting that the benefit cannot be attributed to multi-view supervision alone; how the views are constructed and aggregated is important for preserving transferable reasoning signals. Compared with RLSD, MASS achieves a smaller but consistent advantage, despite relying on a substantially different training mechanism. These results suggest that marginalizing over complementary masked views provides an effective alternative to reward-guided leakage mitigation. The gains persist from smaller to stronger Qwen3 models and transfer to OLMo, indicating that the method is not restricted to a particular model capacity or family.

Figure~\ref{fig:training_dynamics} provides a complementary view of how these differences emerge during training. OPSD typically improves fastest in the early stage, but its performance peaks early and subsequently deteriorates. RLSD exhibits substantially more stable improvement, while MASS initially progresses more slowly and gradually closes the gap as training continues. Around the middle of training, MASS catches up with or crosses RLSD, after which it continues to improve and achieves the strongest late-stage performance. Although the exact crossover point varies across benchmarks, the same qualitative pattern is visible in the averaged trajectory. These dynamics suggest that the advantage of MASS does not come from more aggressive early optimization, but from providing a supervision target that remains useful over longer training horizons.

Motivated by \citet{kim2026does}, we further examine response
length and epistemic token usage as indirect behavioral indicators of
overconfidence during self-distillation. Epistemic tokens refer to expressions
such as \textit{wait}, \textit{perhaps}, and \textit{maybe}, which verbalize
uncertainty and can signal reconsideration of an intermediate reasoning path.
Prior work shows that privileged conditioning tends to suppress such
verbalization, producing shorter and more confident responses; when this
suppression becomes excessive, it can remove signals useful for error
correction and robust reasoning.

As shown in Figure~\ref{fig:epistemic_length}, both OPSD and MASS exhibit an
early increase in response length followed by compression as training
progresses. However, the contraction under OPSD is substantially stronger,
whereas MASS maintains longer reasoning traces and decreases more gradually.
A similar distinction appears in epistemic token usage. OPSD progressively
suppresses epistemic expressions, while MASS shows only a small initial
reduction followed by a gradual recovery, with some variation across
checkpoints. These patterns suggest that MASS mitigates the
confidence-inducing effect of privileged self-distillation: it still learns
to compress reasoning, but does so without aggressively eliminating
uncertainty-aware behavior that may help preserve alternative reasoning
possibilities.

\subsection{Ablation Study}

\paragraph{Anchor-target ablation.}
We examine whether the full-reference teacher should also participate in target construction. In the original MASS, the target is the arithmetic mean of the $K$ masked teachers. In the anchor variant, we instead use
$\widetilde p_t=(p_{0,t}+\sum_{k=1}^{K}p_{k,t})/(K+1)$,
treating the full-reference teacher as one additional view. Since $p_{0,t}$ is fixed with respect to mask sampling, we rescale the estimated target variance by $(K/(K+1))^2$. All other components, including the sampled masks, selective weight, adaptive masking, training schedule, and evaluation protocol, remain unchanged. We conduct this comparison on Qwen3-4B and OLMo-7B using AIME 2025 Avg@12.

\begin{table}[!t]
\centering
\small
\setlength{\tabcolsep}{4.5pt}
\caption{\textbf{Ablation on the anchor teacher.}
AIME 2025 Avg@12 accuracy is reported.}
\label{tab:anchor_ablation}
\begin{tabular}{llcc}
\toprule
Model & Target & Acc. \\
\midrule
Qwen3-4B
& \textbf{Masked avg. (Base)}
& \textbf{73.89} \\
& Masked avg. $+$ anchor
& 72.50 \\

\addlinespace[0.35\baselineskip]

OLMo-7B
& \textbf{Masked avg. (Base)}
& \textbf{75.00}\\
& Masked avg. $+$ anchor
& 73.89 \\
\bottomrule
\end{tabular}
\end{table}

As shown in Table~\ref{tab:anchor_ablation}, adding the anchor to the target produces a modest decrease on both models. The degradation remains limited because the $K$ masked teachers still contribute most of the averaged target, while the anchor can provide a stable and confident signal. Nevertheless, its inclusion shifts supervision back toward the fully observed reference and dilutes continuations supported only by masked views. The consistent decrease suggests that the full-reference teacher is more effective as a signal for selecting useful corrections than as a direct component of the distillation target.

\begin{figure*}[t]
    \centering
    \begin{subfigure}[t]{0.49\linewidth}
        \centering
        \includegraphics[width=\linewidth]{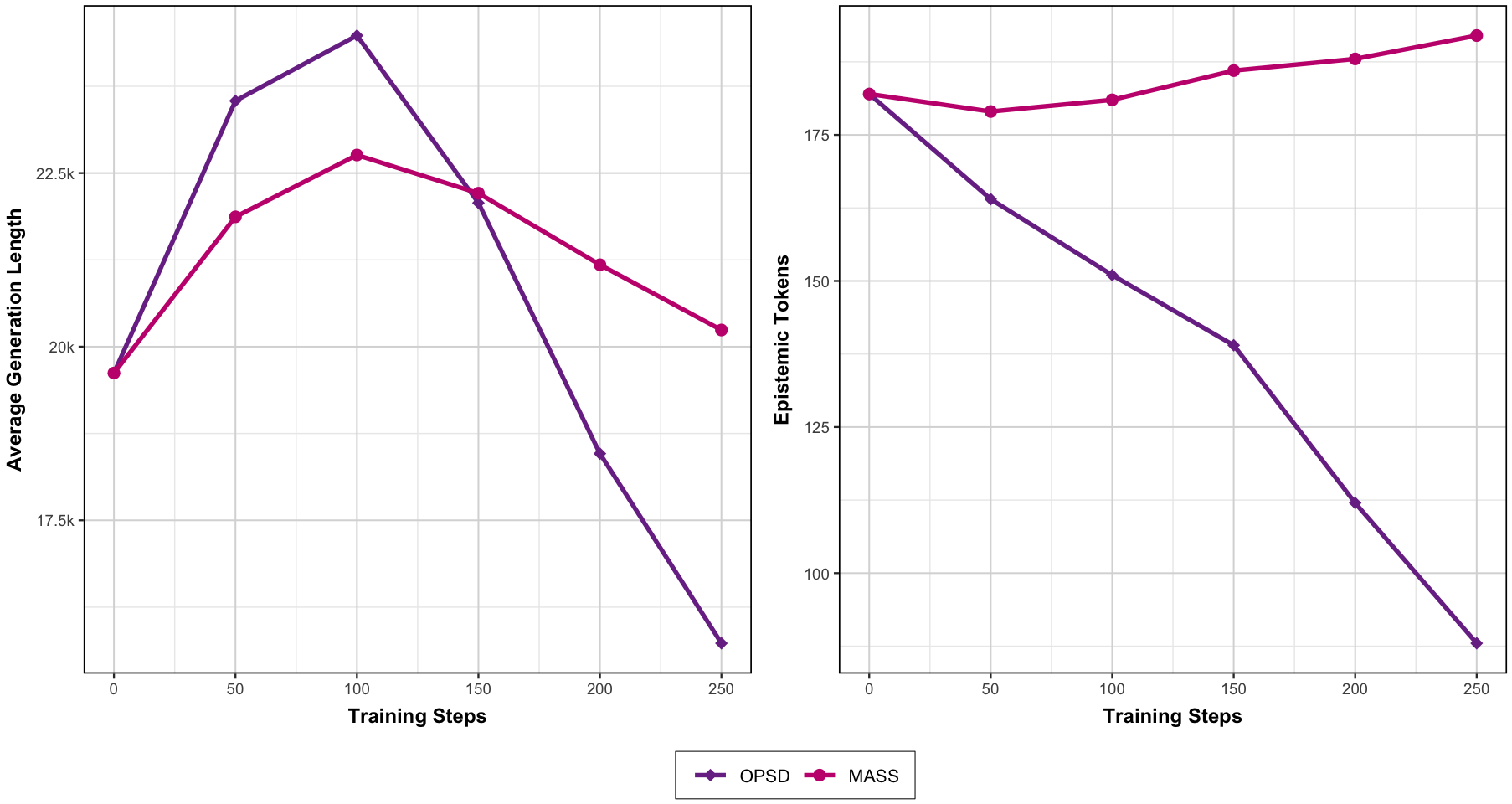}
        \caption{Qwen3-8B}
        \label{fig:epistemic_qwen8b}
    \end{subfigure}
    \hfill
    \begin{subfigure}[t]{0.49\linewidth}
        \centering
        \includegraphics[width=\linewidth]{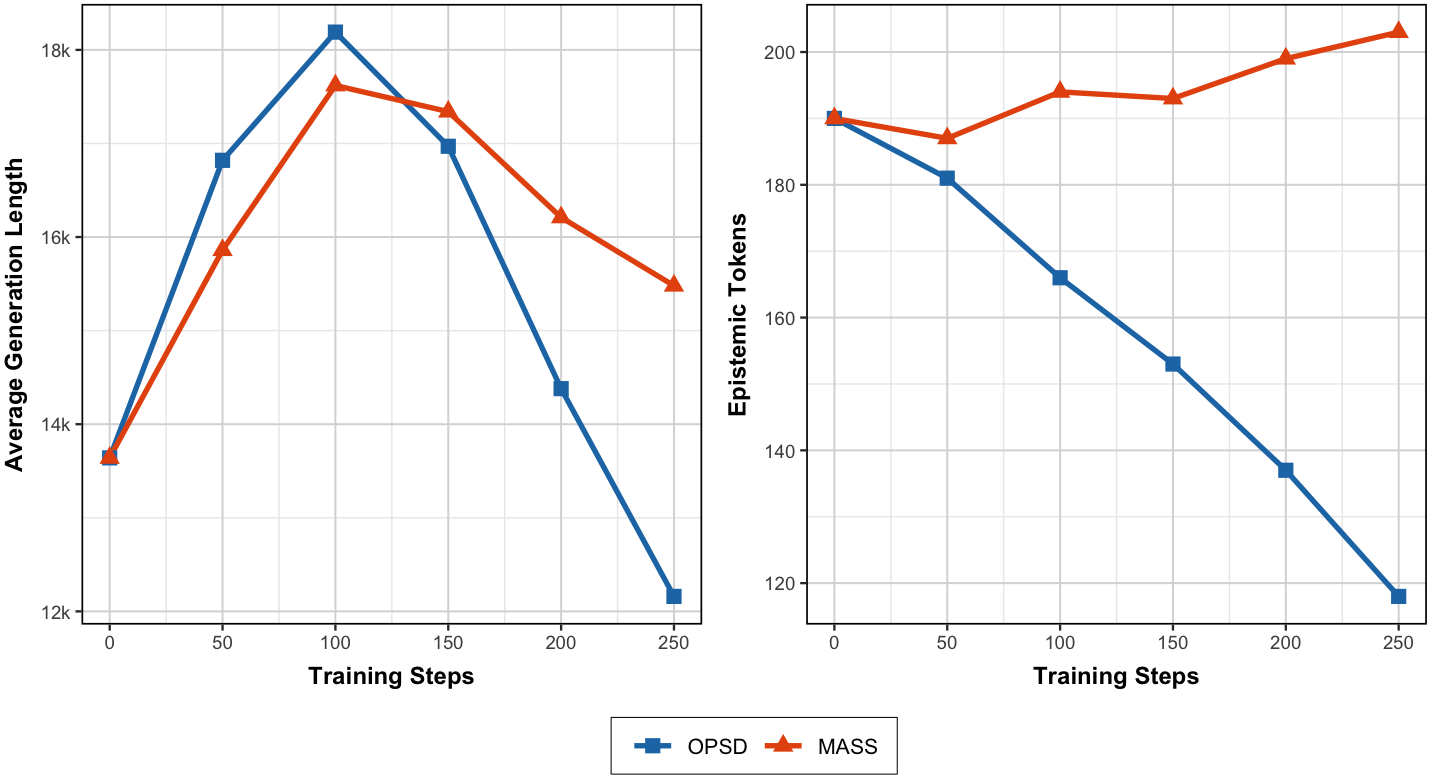}
        \caption{OLMo-7B}
        \label{fig:epistemic_olmo7b}
    \end{subfigure}

    \caption{
   \textbf{Reasoning behavior during self-distillation.} We measured average response length and epistemic token usage for (a) Qwen3-8B and (b) OLMo-7B on AIME25 evaluation.
    }
    \label{fig:epistemic_length}
\end{figure*}

\paragraph{Sensitivity to Teacher Route Scaling and Masking Strength.}
We then study how MASS scales with the number of teacher routes. As shown in
Figure~\ref{fig:teacher_routes}, performance improves as more
masked teacher routes are introduced. With only a small number of routes, the
aggregated target is still determined by a limited set of privileged views.
Increasing the number of routes provides broader coverage over complementary
masked contexts, reducing dependence on any particular privileged trajectory
and allowing the aggregated teacher to preserve a richer set of valid
reasoning signals. The gains become more pronounced as route coverage
increases, suggesting that sufficiently diverse teacher routes are important
for effective marginalization.
\begin{figure}[H]
    \centering
    \includegraphics[width=0.7\linewidth]{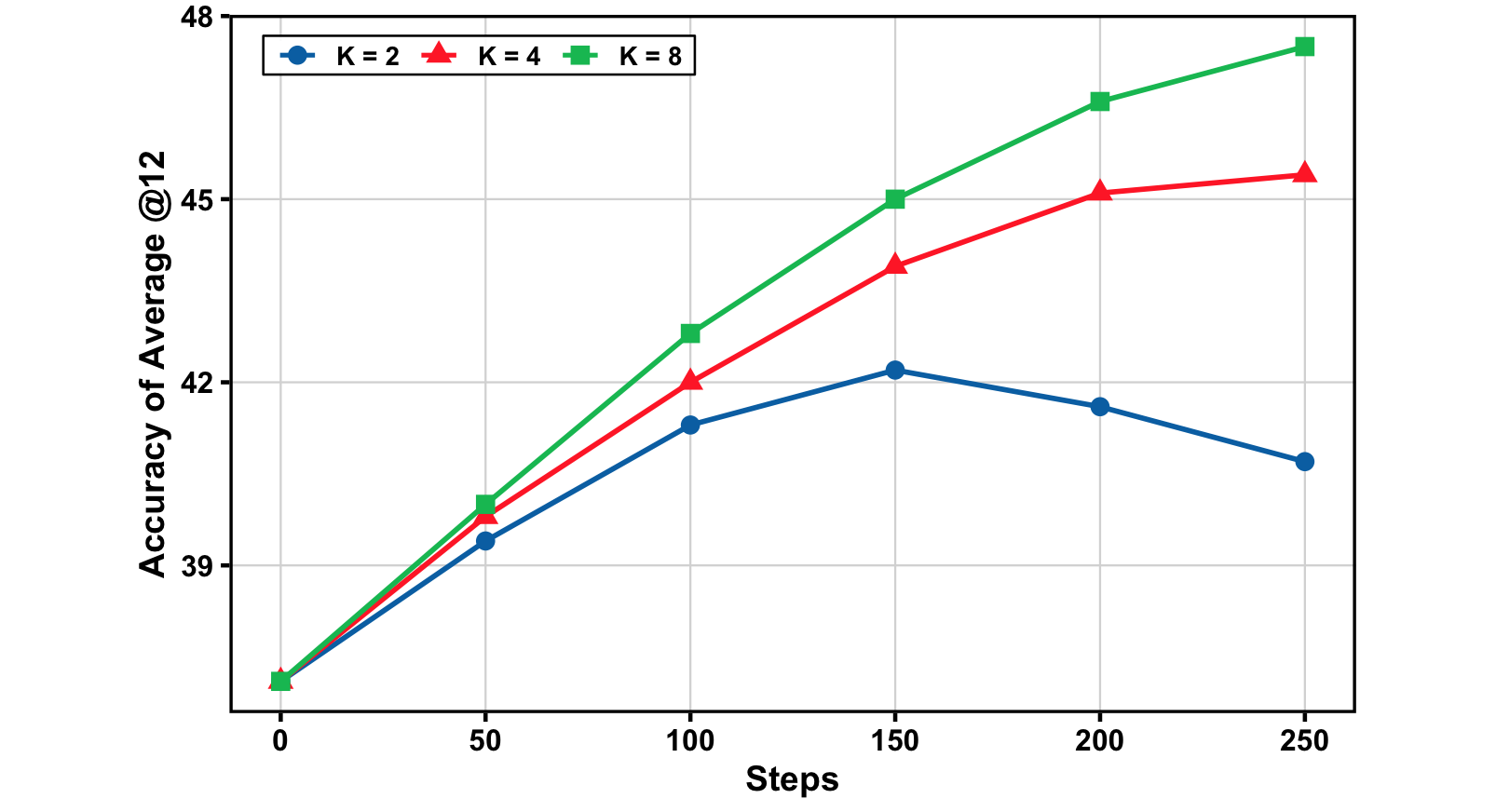}
    \caption{\textbf{Sensitivity to the number of teacher routes.}
    Increasing the number of routes improves AIME 2025 Avg@12 accuracy on
    Qwen3-1.7B.}
    \label{fig:teacher_routes}
\end{figure}
We next vary the masking ratio while keeping the remaining training setup
fixed. Figure~\ref{fig:masking_ratio} exhibits an inverted-U
pattern, with a masking ratio of $0.3$ achieving the strongest performance.
When the masking ratio is too small, most of the reference solution remains
visible to the teacher, leaving the supervision strongly privileged and
insufficiently mitigating overconfident, trajectory-specific guidance.
Conversely, overly aggressive masking removes too much useful reasoning
evidence and weakens the informativeness of the teacher target. The intermediate
masking ratio provides the best trade-off between reducing
privileged-information bias and retaining sufficient supervision signal.
\begin{figure}[H]
    \centering
    \includegraphics[width=0.7\linewidth]{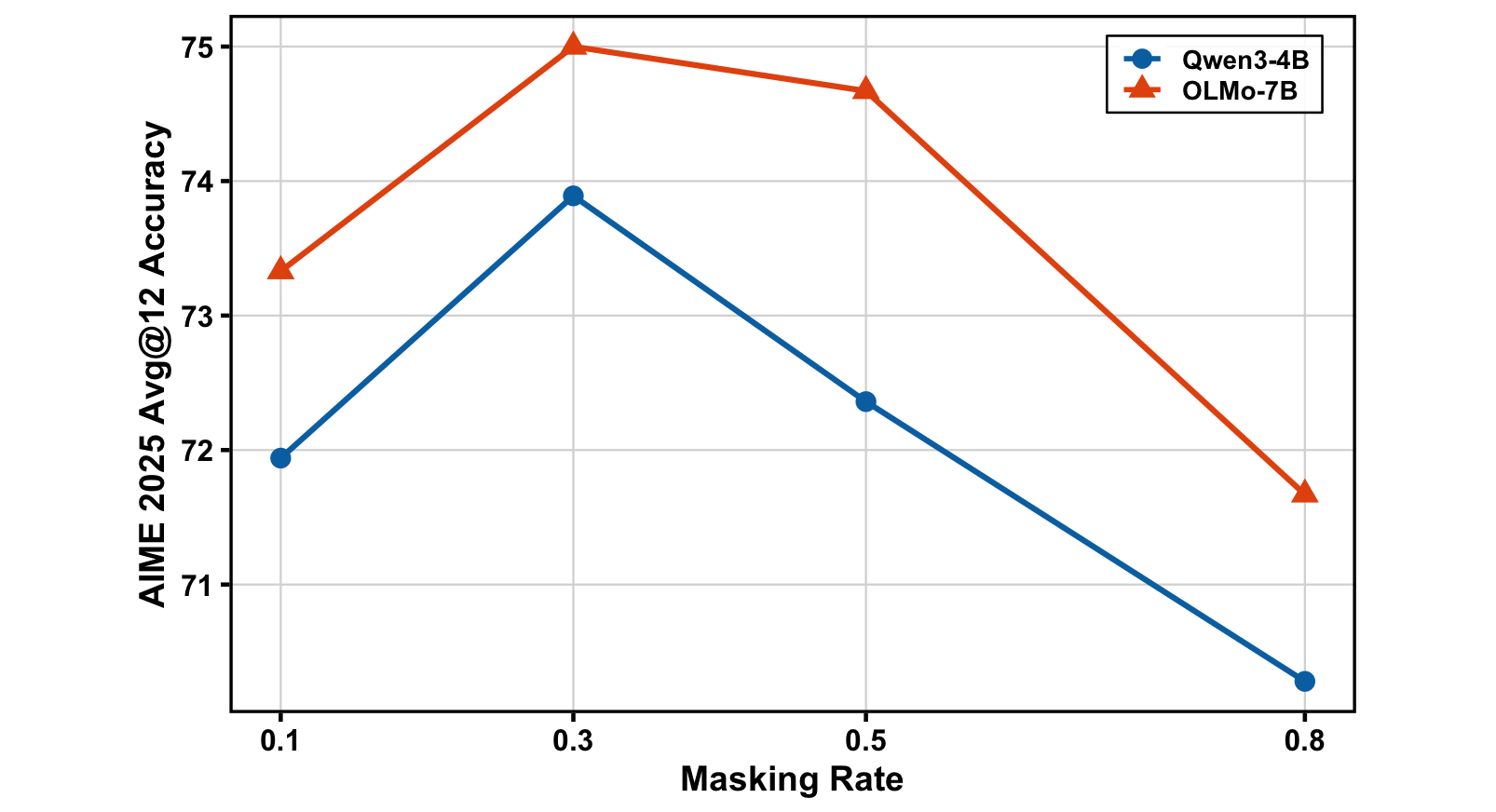}
    \caption{\textbf{Sensitivity to masking strength.}
    On Qwen3-4B, a moderate masking ratio gives the best AIME 2025 Avg@12
    accuracy.}
    \label{fig:masking_ratio}
\end{figure}

\paragraph{Reliability Weighting and Adaptive Masking.}
To isolate the contribution of the two components in MASS, we ablate
reliability weighting and adaptive masking on Qwen3-4B and evaluate on
AIME 2025. We compare three reduced variants against the full model:
removing both components, retaining only reliability weighting, and
retaining only adaptive masking.

\begin{table}[H]
    \centering
       \caption{
    \textbf{Ablation of reliability weighting and adaptive masking.}
    Both components improve AIME 2025 performance on Qwen3-4B, with the full
    configuration achieving the strongest result.
    }
    \setlength{\tabcolsep}{3pt}
    \begin{tabular}{lccc}
        \toprule
        Variant & Weighting & Adaptive Masking & AIME25 \\
        \midrule
        Neither & \xmark & \xmark & 70.83 \\
        Weighting only & \cmark & \xmark & 72.22 \\
        Adaptive only & \xmark & \cmark & 72.78 \\
        Full MASS & \cmark & \cmark & \textbf{73.89} \\
        \bottomrule
    \end{tabular}
 
    \label{tab:weight-mask-ablation}
\end{table}

The results show that both components contribute independently and are
complementary. Reliability weighting reduces the influence of less reliable
masked teacher signals, while adaptive masking improves the construction of
informative masked views during training. Adaptive masking provides a slightly
larger gain when used alone, but the strongest performance is obtained when
both mechanisms are enabled. This suggests that MASS benefits from jointly
controlling \emph{which} masked views are constructed and \emph{how} their
supervision is weighted.

\paragraph{Out-of-Distribution (OOD) Generalization.}
To examine whether the improvements of MASS transfer beyond the mathematical reasoning domain used for training, we conduct an out-of-distribution evaluation on GPQA\citep {rein2023gpqa} and BBH \citep{suzgun2023challenging}. These benchmarks cover scientific knowledge and diverse reasoning tasks, respectively, and therefore differ from our mathematical training distribution. We evaluate the Qwen3-8B and OLMo3-7B-Think models and compare MASS with the corresponding base models, OPSD, and AVSD.
As shown in Table~\ref{tab:ood_results}, MASS achieves the strongest performance across both benchmarks and model families. In contrast, OPSD and AVSD generally reduce out-of-distribution accuracy, suggesting that conventional privileged self-distillation can over-specialize the model to the training-domain reasoning patterns. By marginalizing over complementary masked references and selectively retaining reliable corrections, MASS may reduce this specialization and preserve more of the models' general reasoning capabilities.
\begin{table}[t]
\centering
\small
\caption{\textbf{OOD evaluation.}
GPQA Avg@8 and BBH Avg@1 accuracy are reported.}
\label{tab:ood_results}
{\fontsize{8}{8.5}\selectfont
\setlength{\tabcolsep}{4.5pt}
\renewcommand{\arraystretch}{0.82}

\begin{tabular}{llcc}
\toprule
Model & Method & GPQA & BBH \\
\midrule
\multirow{4}{*}{Qwen3-8B}
& Base & 57.7 & 84.4 \\
& OPSD & 54.8 & 82.9 \\
& AVSD & 56.6 & 83.7 \\
& \textbf{MASS} & \textbf{58.3} & \textbf{84.6} \\
\midrule
\multirow{4}{*}{OLMo3-7B-Think}
& Base & 46.2 & 86.6 \\
& OPSD & 43.9 & 85.0 \\
& AVSD & 44.1 & 85.9 \\
& \textbf{MASS} & \textbf{46.9} & \textbf{86.8} \\
\bottomrule
\end{tabular}
}
\end{table}

\section{Related Work}
\label{sec:related-work}

\paragraph{Knowledge Distillation and On-Policy Self-Distillation.}
Knowledge distillation transfers predictive knowledge from a capable teacher to a smaller student \citep{hinton2015distilling,wang2020minilm,hsieh2023distilling}. Conventional off-policy distillation relies on fixed teacher demonstrations and can therefore suffer from a mismatch with student-induced states \citep{ross2011reduction}. On-policy distillation instead evaluates the teacher along student trajectories \citep{agarwal2024policy}, while on-policy self-distillation uses the same model as both student and teacher, with the latter receiving privileged information such as a verified solution \citep{zhao2026self}. This efficient framework has also been extended to rubrics, visual inputs, and long-context information \citep{ye2026policy,zhang2026opsdl,bablani2026rubrics}. Its central challenge is whether supervision produced with privileged information remains transferable when that information is unavailable at inference time.

\paragraph{Distillation from Privileged Information.}
Recent studies identify several consequences of privileged conditioning, including reference-dependent signals, unstable optimization, trajectory-specific preferences, style drift, and suppressed epistemic behavior \citep{yang2026self,harne2026privileged,pan2026rlcsd,kim2026does}. Existing methods address these effects by anchoring or contrasting teacher signals \citep{yang2026self,pan2026rlcsd,he2026not}, or by interpolating, reshaping, and selectively applying the privileged target \citep{li2026demopsd,xu2026betaopsdderivingpolicyoptimization,shen2026anti}. PAINT varies partial-solution exposure \citep{tan2026paint}, whereas AVSD combines multiple types of privileged information through cross-view consensus \citep{nguyen2026avsd}. MASS focuses on the setting where only one verified solution is available. It creates complementary views through independent span masking and averages their predictions as an empirical masked-reference marginal. Selective weighting and adaptive masking further control finite-view uncertainty and focus masking on influential spans, allowing MASS to broaden the teacher target rather than only rescale supervision from a fixed privileged view.

\section{Conclusion}
We introduced \textbf{Multi-Mask Selective Self-Distillation (MASS)}, which combines multi-mask teacher aggregation, selective distillation, and adaptive masking to outperform existing baselines across models and reasoning benchmarks. Beyond these gains, our results suggest that constructing complementary teacher contexts can broaden support beyond the demonstrated trajectory, while reliability aware weighting separates useful view diversity from finite mask uncertainty. More broadly, MASS highlights teacher-context design as an important dimension of privileged self-distillation. Future work may incorporate MASS-derived token-level correction signals into RLVR to provide dense, reliable advantage weighting within an GRPO framework.

Several limitations remain. While MASS requires no additional autoregressive rollouts and is substantially more token-efficient than group-based generation, evaluating multiple masked teacher routes still increases peak accelerator-memory usage. This overhead may be reduced through memory-aware route batching or adaptive allocation of teacher routes. The current implementation also relies on fixed span segmentation and sampled Shapley attribution, motivating learned segmentation and more efficient attribution estimators. Finally, our evaluation mainly focuses on reasoning. Extending MASS to other domains such as coding remains an important direction.
\section{Acknowledgement}
This research was initiated through the UC Santa Barbara Research Mentorship Program (RMP), which we gratefully acknowledge. We thank Dr. Lina Kim and Dr. Teresa Holden for directing the program. The work begun during the program was subsequently continued and extended to produce this paper.
\section{Reproducibility Statement}
To facilitate reproducibility, we provide detailed descriptions of the model architecture in Section~\ref{sec:method}, the training and evaluation strategy in Section~\ref{subsec::eval}, and the hyperparameters and other training details in Appendix~\ref{experimentdetail}. The source code used in our experiments is publicly available at \url{https://github.com/researcheranoymous-oss/MASS_2027}

\bibliographystyle{plainnat}
\bibliography{references}

\appendix

\section{Theoretical Development}
\label{app:theory}

This appendix collects the derivations that Section~2 states without proof and makes explicit the assumptions behind them. Throughout, we fix a student prefix $s_t$, the model snapshot, and the masking distribution $q$, and drop the position index $t$ when no confusion arises. All probability vectors live in the simplex $\simplex$ over the vocabulary $\mathcal V$, and $\langle\cdot,\cdot\rangle$, $\|\cdot\|$ denote the Euclidean inner product and norm on $\mathbb R^{|\mathcal V|}$. The $K$ masks $m^{(1)},\dots,m^{(K)}$ are i.i.d.\ draws from $q$, so the masked teachers $p_k=p_T(\cdot\mid s_t,g_{m^{(k)}}(R))$ are i.i.d.\ random vectors with mean $\pbar=\E_{M\sim q}[p_M]$ and covariance $\Sigma=\Cov_{M\sim q}(p_M)$. We write $\phat=\frac1K\sum_k p_k$, $g=\pS-\pbar$ and $\ghat=\pS-\phat$ as in Section~2.2.

\subsection{Forward-KL distillation acts on the logits through the target gap}

\begin{lemma}[Logit gradient of forward KL]
\label{lem:grad}
Let $\pS=\mathrm{softmax}(z)$ for logits $z\in\mathbb R^{|\mathcal V|}$ and let $p\in\simplex$ be fixed. Then $\nabla_z \KL(p\,\|\,\pS)=\pS-p$.
\end{lemma}
\begin{proof}
$\KL(p\|\pS)=\sum_v p_v\log p_v-\sum_v p_v z_v+\log\sum_v e^{z_v}$. The first term is constant in $z$, the second has gradient $-p$, and the third has gradient $\mathrm{softmax}(z)=\pS$.
\end{proof}

Applied to the MASS objective in Eq.~(6), Lemma~\ref{lem:grad} gives
\begin{equation}
\label{eq:app-grad}
\nabla_{z_t}\mathcal L(\theta)=\tfrac1T\,\what_t\,(p_{S,t}-\phat_t)=\tfrac1T\,\what_t\,\ghat_t,
\end{equation}
so the per-position update is exactly the empirical mask-marginal gradient $\ghat_t$ scaled by the selective weight. The weight in Section~2.2 is therefore best read as controlling the error of $\what_t\ghat_t$ as an estimator of the population correction $g_t$.

\subsection{What the mask marginal removes}

The introduction motivates the arithmetic mean of masked teachers as a surrogate for a marginal over reference realisations. The following identity makes precise which part of the masked-teacher signal is discarded by distilling toward $\pbar$ instead of toward individual views.

\begin{proposition}[Decomposition of masked-teacher distillation]
\label{prop:decomp}
For any $r\in\simplex$ with full support,
\[
\E_{M\sim q}\KL(p_M\,\|\,r)=\KL(\pbar\,\|\,r)+\E_{M\sim q}\KL(p_M\,\|\,\pbar).
\]
Consequently, $\pbar$ is the unique minimiser of $r\mapsto\E_M\KL(p_M\|r)$, and the second term equals the mutual information $I(M;Y)$ between the mask $M\sim q$ and the next token $Y\sim p_M$.
\end{proposition}
\begin{proof}
Expand the log-ratio as $\E_M\sum_v p_M(v)\log\frac{p_M(v)}{\pbar(v)}+\E_M\sum_v p_M(v)\log\frac{\pbar(v)}{r(v)}$. In the second term $\E_M p_M(v)=\pbar(v)$, which yields $\KL(\pbar\|r)$. Uniqueness follows because the first term does not depend on $r$ and $\KL(\pbar\|r)\ge0$ with equality iff $r=\pbar$. The identification with $I(M;Y)$ is the definition of mutual information for the joint law $q(M)\,p_M(Y)$, whose $Y$-marginal is $\pbar$.
\end{proof}

Distilling toward $\pbar$ therefore equals average single-view distillation minus $I(M;Y)$, the part of the teacher's prediction explained by \emph{which} spans happened to be visible. When $q$ is degenerate at the empty mask, $I(M;Y)=0$ and MASS reduces to OPSD; otherwise exactly the mask-specific variation is removed from the target, which is the sense in which the marginal is ``less reference-specific'' than any single view.

The choice of arithmetic rather than agreement-seeking (geometric) pooling is a support statement:

\begin{proposition}[Arithmetic pooling preserves partially supported continuations]
\label{prop:support}
Let $p_1,\dots,p_K\in\simplex$ and $\phat=\frac1K\sum_k p_k$.
\begin{enumerate}
\item[(i)] For every $v$, $\phat(v)\ge\frac1K\max_k p_k(v)$ and $\log\phat(v)\ge\frac1K\sum_k\log p_k(v)$.
\item[(ii)] The normalised geometric pool $p^{\mathrm{geo}}(v)\propto\prod_k p_k(v)^{1/K}$ satisfies $p^{\mathrm{geo}}(v)=0$ whenever $p_k(v)=0$ for some $k$, and its unnormalised mass is bounded by $\phat(v)$.
\item[(iii)] For a fixed continuation $c=(c_1,\dots,c_L)$ scored at fixed prefixes, the PI Bias Score of the arithmetic teacher satisfies $\mathrm{PS}_{\phat}(c)\ge\frac1K\sum_k\mathrm{PS}_{p_k}(c)$.
\end{enumerate}
\end{proposition}
\begin{proof}
(i) The first inequality drops all but the largest term of the sum; the second is Jensen's inequality for the concave logarithm. (ii) A zero factor annihilates the product; the bound is the arithmetic--geometric mean inequality. (iii) $\mathrm{PS}(c)=\frac1L\sum_l[\log p_{T,l}(c_l)-\log p_{S,l}(c_l)]$; apply (i) position-wise, with the student term common to all views.
\end{proof}

Part (iii) guarantees that the mixed teacher never scores a fixed continuation below the average single-view score. It does not predict that support for a \emph{different} correct solution rises faster than for the in-context one; that is the empirical content of Figure~1.

\subsection{Finite-mask statistics and the bias-corrected estimators}

\begin{lemma}[Moments of the empirical teacher]
\label{lem:moments}
$\E[\phat]=\pbar$, $\Cov(\phat)=\Sigma/K$, $\E[\ghat]=g$, and $V:=\E\|\ghat-g\|^2=\tr(\Sigma)/K\le2/K$.
\end{lemma}
\begin{proof}
The first three statements are the mean and covariance of an average of $K$ i.i.d.\ vectors. For the bound, $\tr\Sigma=\E\|p_M-\pbar\|^2$, and for $p,p'\in\simplex$, $\|p-p'\|_2^2\le\|p-p'\|_1\|p-p'\|_\infty\le2$.
\end{proof}

\begin{proposition}[Conditional unbiasedness of Eq.~(4)]
\label{prop:unbiased}
With $\Vhat=\frac1{K(K-1)}\sum_k\|p_k-\phat\|^2$, $\Shat=\|\pS-\phat\|^2-\Vhat$ and $\Bhat=\langle\pS-\phat,p_0-\phat\rangle-\Vhat$,
\[
\begin{gathered}
\E[\Vhat]=V,\qquad \E[\Shat]=S=\|g\|^2,\\ \E[\Bhat]=B=\langle\pS-\pbar,\,p_0-\pbar\rangle .
\end{gathered}
\]
\end{proposition}
\begin{proof}
Write $p_k=\pbar+\xi_k$ with $\E\xi_k=0$, $\Cov\xi_k=\Sigma$, and $\bar\xi=\frac1K\sum_k\xi_k$, so $\phat=\pbar+\bar\xi$.

\emph{Variance.} $\sum_k\|p_k-\phat\|^2=\sum_k\|\xi_k\|^2-K\|\bar\xi\|^2$. Taking expectations, $K\tr\Sigma-K\cdot\tr\Sigma/K=(K-1)\tr\Sigma$, hence $\E\Vhat=\tr\Sigma/K=V$.

\emph{Signal.} $\|\pS-\phat\|^2=\|g-\bar\xi\|^2=\|g\|^2-2\langle g,\bar\xi\rangle+\|\bar\xi\|^2$, whose expectation is $S+V$. Subtracting $\Vhat$ removes the $V$ term in expectation.

\emph{Anchor.} $\langle\pS-\phat,p_0-\phat\rangle=\langle g-\bar\xi,\,(p_0-\pbar)-\bar\xi\rangle=B-\langle g,\bar\xi\rangle-\langle p_0-\pbar,\bar\xi\rangle+\|\bar\xi\|^2$; the cross terms vanish in expectation (note $p_0$ is deterministic given the state) and $\E\|\bar\xi\|^2=V$, so subtracting $\Vhat$ again yields $B$.
\end{proof}

Both uncorrected statistics carry the same $+V$ because each is a quadratic form in which $\phat$ appears twice, which is why a single $\Vhat$ corrects both.

\subsection{The selective weight as a shrinkage estimator}

\begin{proposition}[Optimal scalar weight]
\label{prop:weight}
Fix $\lambda\ge0$, $\varepsilon>0$, and consider the population criterion
$J(w)=\E\|w\ghat-g\|^2+\varepsilon w^2-2\lambda w[B]_+$ over $w\in[0,1]$. Then $\E\|w\ghat-g\|^2=(1-w)^2S+w^2V$, $J$ is strictly convex, and its constrained minimiser is
\[
w^\star=\min\Big\{1,\ \frac{S+\lambda[B]_+}{S+V+\varepsilon}\Big\}.
\]
$w^\star$ is nondecreasing in $S$ and in $[B]_+$ and nonincreasing in $V$.
\end{proposition}
\begin{proof}
Since $\ghat=g+(\ghat-g)$ with $\E[\ghat-g]=0$, $\E\|w\ghat-g\|^2=\|(w-1)g\|^2+w^2\E\|\ghat-g\|^2=(1-w)^2S+w^2V$. Hence $J(w)=(1-w)^2S+w^2(V+\varepsilon)-2\lambda w[B]_+$ with $J''=2(S+V+\varepsilon)>0$. Setting $J'(w)=-2(1-w)S+2w(V+\varepsilon)-2\lambda[B]_+=0$ gives the unconstrained minimiser $(S+\lambda[B]_+)/(S+V+\varepsilon)$, which is nonnegative; projecting a strictly convex one-dimensional problem onto $[0,1]$ clips it at $1$. Monotonicity follows by differentiating the ratio.
\end{proof}

\begin{corollary}[Mean-squared-error gain over unweighted distillation]
\label{cor:mse}
With $\lambda=0$ and $\varepsilon\to0$, $\min_w\E\|w\ghat-g\|^2=\frac{SV}{S+V}\le\min\{S,V\}$. Relative to the unweighted update $w=1$ (error $V$), the optimal weight reduces the error by the factor $S/(S+V)$; relative to no update $w=0$ (error $S$), by $V/(S+V)$.
\end{corollary}
\begin{proof}
Substitute $w^\star=S/(S+V)$ into $(1-w)^2S+w^2V$.
\end{proof}

When $V\gg S$ the error is dominated by mask sampling and $w^\star\approx S/V\ll1$; when $S\gg V$ the correction is reproducible and $w^\star\approx1$. Since $V=\tr\Sigma/K$, more masks move every position toward $w^\star=1$. The anchor term enters $J$ linearly, so it only raises the weight by $\lambda[B]_+/(S+V+\varepsilon)$: $B>0$ means the student and the full-reference teacher lie on the same side of the marginal, and the extra credit prevents such a correction from being suppressed when $V$ is moderate. It is a statement about the geometry of three distributions, not about correctness.

\begin{remark}[Plug-in weight]
\label{rem:plugin}
Eq.~(5) replaces $(S,V,B)$ by $([\Shat]_+,\Vhat,[\Bhat]_+)$. Three departures from Proposition~\ref{prop:weight} follow. (a) A ratio of unbiased estimators is not unbiased, and the clipping $[\cdot]_+$ introduces a nonnegative bias when the true $S$ or $B$ is near zero. (b) $\what$ and $\ghat$ share the same $K$ masks, so $\E[\what\ghat]\ne w^\star g$ in general; sample splitting across masks would restore this at the cost of a noisier target. (c) Consistency: as $K\to\infty$, $\Shat\to S$, $\Bhat\to B$ and $K\Vhat\to\tr\Sigma$ almost surely by the strong law, so $\Vhat\to0$ and $\what\to\min\{1,(S+\lambda[B]_+)/(S+\varepsilon)\}$ by the continuous mapping theorem. With infinitely many masks the marginal is known exactly and no shrinkage is applied, as intended.
\end{remark}

\subsection{Shapley attribution from the distillation masks}

Let $[J]=\{1,\dots,J\}$, let $v$ be the coalition utility of Eq.~(7) with $v(\emptyset)=0$, and let $\pi(A)=\frac{J-1}{\binom{J}{|A|}|A|(J-|A|)}$ for $0<|A|<J$ be the Shapley kernel. We use the classical characterisation of Shapley values as a weighted least-squares fit (Charnes et al., 1988; Lundberg \& Lee, 2017): for any set function $v$,
\begin{equation}
\label{eq:kernelshap}
\begin{split}
\phi={}&\arg\min_{\beta\in\mathbb R^J}\ \sum_{0<|A|<J}\pi(A)\Big(v(A)-\sum_{j\in A}\beta_j\Big)^2\\
&\text{s.t.}\quad \sum_j\beta_j=v([J]),
\end{split}
\end{equation}
where $\phi$ is the vector of Shapley values in Eq.~(8). Note that this holds for arbitrary (non-additive) $v$: the linear model is a device for computing $\phi$, not an assumption on the game.

\begin{proposition}[Importance-corrected regression targets the Shapley objective]
\label{prop:shap}
Let $A_1,\dots,A_K$ be the coalitions of the i.i.d.\ masks $m^{(k)}\sim\mathrm{Bernoulli}(q)$, let $\Pr_q(A)=\prod_{j\in A}q_j\prod_{j\notin A}(1-q_j)$, and set $\omega_k=\pi(A_k)/\Pr_q(A_k)$ for $0<|A_k|<J$ (samples with $A_k\in\{\emptyset,[J]\}$ are discarded, their values being fixed by the constraints). Then for every $\beta$,
\[
\begin{aligned}
&\E\Big[\sum_{k}\omega_k\big(R^{(k)}-m^{(k)\top}\beta\big)^2\Big]\\
&\qquad=K\sum_{0<|A|<J}\pi(A)\Big(v(A)-\sum_{j\in A}\beta_j\Big)^2 .
\end{aligned}
\]
Under the box constraint $\epsilon\le q_j\le1-\epsilon$ every coalition has $\Pr_q(A)\ge\epsilon^J>0$, so the weights are finite and bounded by $\pi(A)\epsilon^{-J}$. If $\tau>0$ the objective in Eq.~(9) is strictly convex and its constrained minimiser $\widehat\phi$ is unique; as $K\to\infty$ with $\tau$ fixed, $\widehat\phi\to\phi$ almost surely.
\end{proposition}
\begin{proof}
Each summand has expectation $\sum_A\Pr_q(A)\cdot\frac{\pi(A)}{\Pr_q(A)}(v(A)-\mathbf 1_A^\top\beta)^2$, which is the Shapley-kernel objective; sum over $k$. Strict convexity follows from the added $\tau\|\beta-z\|^2$ term. For consistency, divide Eq.~(9) by $K$: by the strong law the data term converges to the objective in Eq.~\eqref{eq:kernelshap}, the ridge term vanishes as $\tau/K\to0$, and the limit is a strictly convex quadratic on the constraint set (its design matrix over all interior coalitions has full column rank for $J\ge2$). Minimisers of strictly convex quadratics depend continuously on their coefficients, so $\widehat\phi\to\phi$.
\end{proof}

In practice Eq.~(9) is solved through its KKT system $(\mathbf M^\top\Omega\mathbf M+\tau I)\beta+\nu\mathbf 1=\mathbf M^\top\Omega\mathbf R+\tau z$, $\mathbf 1^\top\beta=v([J])$; when $K<J$ the data term is rank-deficient and $\tau>0$ is what makes $\widehat\phi$ well defined, with $z$ fixing the components the current masks do not identify.

\subsection{The adaptive masking update in closed form}

\begin{proposition}[Solution of Eq.~(10)]
\label{prop:adapt}
Assume $\epsilon\le\rho\le1-\epsilon$ so that $\mathcal Q\ne\emptyset$, and let $I_j\ge0$, $\eta>0$. The objective in Eq.~(10) is strictly concave on $(0,1)^J$ and separable, so its maximiser on $\mathcal Q$ is unique and given by
\[
q^{(n+1)}_j=\mathrm{clip}_{[\epsilon,1-\epsilon]}\Big(\sigma\Big(\frac{\logit q^{(n)}_j-\eta\mu}{1+\eta I_j}\Big)\Big),
\]
where $\sigma$ is the logistic function, $\logit=\sigma^{-1}$, and $\mu\in\mathbb R$ is the unique scalar for which $\sum_jq^{(n+1)}_j=\rho J$; $\mu$ can be found by bisection because the left-hand side is continuous and nonincreasing in $\mu$.
\end{proposition}
\begin{proof}
Write $F(q)=\sum_jF_j(q_j)$ with $F_j(u)=I_jH_{\mathrm{Ber}}(u)-\frac1\eta\KLBer(u\|q^{(n)}_j)$. Binary entropy and $-\KLBer(\cdot\|c)$ are strictly concave, so $F$ is strictly concave on the compact convex set $\mathcal Q$, giving existence and uniqueness. Using $H'_{\mathrm{Ber}}(u)=-\logit u$ and $\frac{d}{du}\KLBer(u\|c)=\logit u-\logit c$,
\[
F_j'(u)=-I_j\logit u-\tfrac1\eta\big(\logit u-\logit q^{(n)}_j\big).
\]
For a separable concave objective with one linear equality and box constraints, the KKT conditions state that $q_j$ is the projection onto $[\epsilon,1-\epsilon]$ of the solution of $F_j'(u)=\mu$, where $\mu$ is the multiplier of the budget (if the unconstrained root lies outside the box, $F_j'$ is monotone so the constraint that is violated is active). Solving $F_j'(u)=\mu$ gives $(I_j+\tfrac1\eta)\logit u=\tfrac1\eta\logit q^{(n)}_j-\mu$, i.e.\ the displayed formula. Each $q_j(\mu)$ is continuous and nonincreasing in $\mu$, with $\sum_jq_j(\mu)\to(1-\epsilon)J\ge\rho J$ as $\mu\to-\infty$ and $\to\epsilon J\le\rho J$ as $\mu\to\infty$, so a root exists; uniqueness of the maximiser makes any choice among tied roots immaterial.
\end{proof}

\begin{corollary}[Behaviour of the update]
\label{cor:adapt}
\begin{enumerate}
\item[(i)] Spans with $I_j=0$ receive a common shift of their log-odds, $\logit q_j^{(n+1)}=\logit q_j^{(n)}-\eta\mu$: their relative odds are preserved and they absorb whatever budget the influential spans release.
\item[(ii)] Before clipping and up to the common shift, the log-odds of span $j$ are contracted toward $0$ (i.e.\ $q_j$ toward $\tfrac12$) by the factor $1/(1+\eta I_j)$; if the budget constraint is inactive ($\mu=0$) and $I_j$ is held fixed across visits, $q_j^{(n)}\to\tfrac12$ geometrically at rate $1/(1+\eta I_j)$; small $\eta$ yields the proximity behaviour of Section~2.3.
\end{enumerate}
\end{corollary}
\begin{proof}
Read off the affine map on logits in Proposition~\ref{prop:adapt}.
\end{proof}

The entropy reward also serves the attribution: in Proposition~\ref{prop:shap} the regression column for span $j$ has variance $q_j(1-q_j)$, which, like $H_{\mathrm{Ber}}(q_j)$, peaks at $\tfrac12$ and vanishes at the box boundary. Moving an influential span toward $\tfrac12$ therefore maximises the information future masks carry about $\phi_j$, and the box $[\epsilon,1-\epsilon]$ is exactly what keeps every span identifiable in Eq.~(9).

All results above use only that masks are i.i.d.\ from $q$ within an update, that $p_0$ and $\pS$ are fixed given the state (teachers are detached), and that $\epsilon\le q_j\le1-\epsilon$. None requires the coalition model to be additive or $\pbar$ to coincide with a marginal over genuinely distinct solutions; Proposition~\ref{prop:decomp} characterises what the mask marginal removes without claiming more. The finite-$K$ guarantees are conditional on a state, and with the small $K$ used in practice the regularisers $\varepsilon,\tau$ and the clipping in Eq.~(5) are what keep the plug-in quantities stable.

\section{Additional Experimental Results}
\label{app:additional_experiments}
\subsection{Sensitivity Analysis of the Masking Probability Clipping Threshold}

The adaptive masking component in MASS updates the masking probability of each
reference span based on its estimated contribution. To prevent the masking
probabilities from collapsing to deterministic states, we constrain each span
masking probability $q_j$ within a bounded range:
\[
\epsilon \leq q_j \leq 1-\epsilon .
\]
Here, $\epsilon$ controls the degree of clipping applied during adaptive
masking. A smaller value allows more aggressive adaptation but may introduce
instability due to noisy Shapley-based attribution estimates, while a larger
value provides stronger regularization but can limit the flexibility of span
selection. We investigate the effect of this hyperparameter by evaluating
MASS with four different clipping thresholds: $\epsilon \in \{0.01, 0.03, 0.05,
0.10\}$ while keeping all other configurations unchanged.

\begin{table}[H]
\centering
\caption{\textbf{Sensitivity analysis of the masking probability clipping} We report Avg@12 performance on AIME 2025. All experiments use the same MASS training
configuration except for the clipping threshold.}
\label{tab:epsilon_ablation}
\begin{tabular}{c|cccc}
\toprule
Model & $\epsilon=0.01$ & $\epsilon=0.03$ & $\epsilon=0.05$ & $\epsilon=0.10$ \\
\midrule
Qwen3-8B & 79.08 & 79.44 & 79.17 & 78.61 \\
OLMo-7B  & 74.62 & 75.00 & 74.72 & 73.89 \\
\bottomrule
\end{tabular}
\end{table}

As shown in Table~\ref{tab:epsilon_ablation}, MASS achieves the best
performance with $\epsilon=0.03$ on both Qwen3-8B and OLMo-7B. When using a
smaller clipping threshold ($\epsilon=0.01$), the performance slightly
decreases, suggesting that overly flexible masking adaptation may amplify
noise in the estimated span importance and lead to less stable supervision.
Increasing the threshold to $\epsilon=0.05$ results in a minor degradation,
while a more aggressive constraint ($\epsilon=0.10$) causes a larger drop,
indicating that excessive clipping restricts the ability of adaptive masking
to focus on influential reasoning spans. Overall, the results suggest that a
moderate clipping threshold provides a suitable balance between exploration of
alternative masked views and stable adaptive masking.

\subsection{Effect of the KL Divergence Direction}

MASS constructs the distillation target by averaging predictions from multiple
masked teacher views. In the default formulation, we optimize the student with
a forward KL objective:
\[
\mathcal{L}_{\mathrm{FKL}}
=
\frac{1}{T} 
\sum_{t=1}^{T}\mathrm{sg}(\widehat{w}_t)
\mathrm{KL}
(\operatorname{sg}(\hat{p}_{t}) \Vert p_{S,t}),
\]
where $\hat{p}_{t}$ denotes the aggregated masked teacher distribution and
$\operatorname{sg}(\cdot)$ indicates stop-gradient. We investigate whether the
choice of KL direction affects the effectiveness of MASS by comparing forward
KL with the reverse direction:
\[
\mathcal{L}_{\mathrm{BKL}}
=
\frac{1}{T}
\sum_{t=1}^{T} \mathrm{sg}(\widehat{w}_t)
\mathrm{KL}
(p_{S,t}\Vert \operatorname{sg}(\hat{p}_{t})).
\]
All experiments are conducted on Qwen3-4B with the same training setup, and
performance is evaluated on AIME 2025 and AIME 2024 using Avg@12 accuracy.

\begin{table}[H]
\centering
\caption{Effect of KL divergence direction on Qwen3-4B. Results are reported
as Avg@12 accuracy on AIME 2024 and AIME 2025.}
\label{tab:kl_direction}
\begin{tabular}{c|cc}
\toprule
Distillation Objective & AIME24 & AIME25 \\
\midrule
Backward KL & 79.28 & 72.50 \\
Forward KL (MASS) & \textbf{81.11} & \textbf{73.89} \\
\bottomrule
\end{tabular}
\end{table}

As shown in Table~\ref{tab:kl_direction}, forward KL achieves better
performance than backward KL under the MASS framework. We hypothesize that this
difference comes from the different optimization behaviors of the two
objectives. The forward KL objective encourages the student to preserve the
support of the aggregated teacher distribution, which is beneficial when the
teacher target combines complementary reasoning trajectories from multiple
masked views. In contrast, backward KL is more mode-seeking and encourages the
student to concentrate on high-probability regions of the teacher distribution,
which may reduce the diversity preserved by the multi-mask aggregation. These
results support the use of forward KL for transferring the broader supervision
signal constructed by MASS.

\subsection{Prompt Example}
Figure~\ref{fig:prompt_comparison} provides a concrete example of the teacher prompts used in OPSD and MASS to clarify how the two methods differ at the level of privileged conditioning. In both settings, all prompt components other than the reference solution are kept identical, so that the difference in teacher behavior is attributable to how the reference information is exposed. OPSD conditions the teacher on the complete reference solution, whereas MASS constructs multiple teacher views by independently masking spans of the same reference solution; the figure shows one such masked view for illustration. The labels \texttt{[S1]}--\texttt{[S5]} and the surrounding braces are included only in the figure to make the span boundaries and masking units explicit. These span identifiers are \emph{not} included in the actual teacher prompts used during training; in practice, the teacher receives only the original reference text with selected spans replaced by \texttt{[MASK]}.
\begin{figure*}[t]
    \centering
    \includegraphics[width=0.75\textwidth]{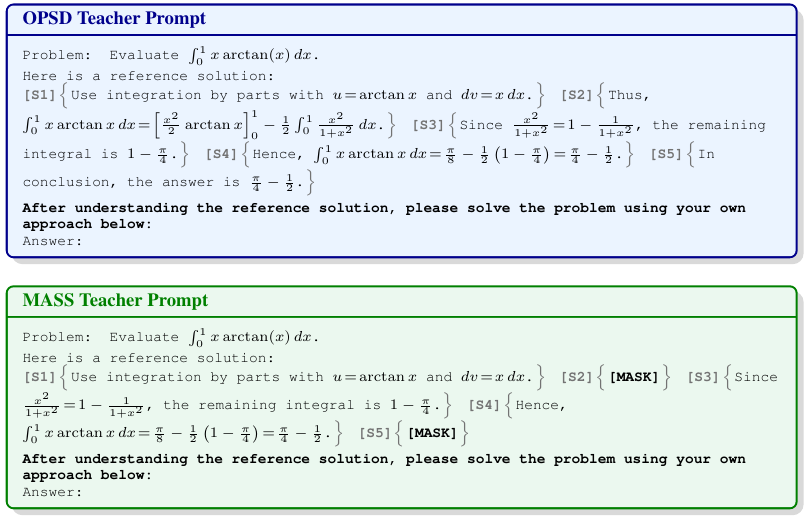}
    \caption{\textbf{Teacher prompts in OPSD and MASS.}
OPSD uses the complete reference solution, while MASS constructs multiple teacher views from independently masked versions of the same reference; one example is shown. Labels \texttt{[S1]}--\texttt{[S5]} and the surrounding braces are added only for visualization to clarify the span boundaries and are not included in the actual prompts. During training, selected spans are directly replaced by \texttt{[MASK]}.}
    \label{fig:prompt_comparison}
\end{figure*}

\subsection{Computational Cost Analysis}
We report the estimated training floating point operations (FLOPs) per optimization step to compare the computational overhead of different training paradigms. We use the standard approximation of $6N$ FLOPs per token for transformer training, where $N$ denotes the number of model parameters. The cost of rollout generation, teacher inference, and student optimization are counted separately.

All methods are estimated under the same Qwen3-4B setting with a batch size of 32 and a maximum response length of 1024 tokens. For MASS, we use $K=8$ masked teacher routes. FLOPs are normalized by OPSD.

\begin{table}[H]
\centering
\caption{Estimated training computation per optimization step. FLOPs are normalized by OPSD.}
\label{tab:flops}
\begin{tabular}{lcc}
\toprule
Method & FLOPs / Step & Relative to OPSD \\
\midrule
OPSD & $1.31 \times 10^{15}$ & $1.00\times$ \\
MASS & $3.41 \times 10^{15}$ & $2.60\times$ \\
RLSD & $8.39 \times 10^{15}$ & $6.40\times$ \\
\bottomrule
\end{tabular}
\end{table}

\subsection{Masking Attribution Heatmap}
\begin{figure*}[!t]
    \centering
    \includegraphics[width=0.9\textwidth]{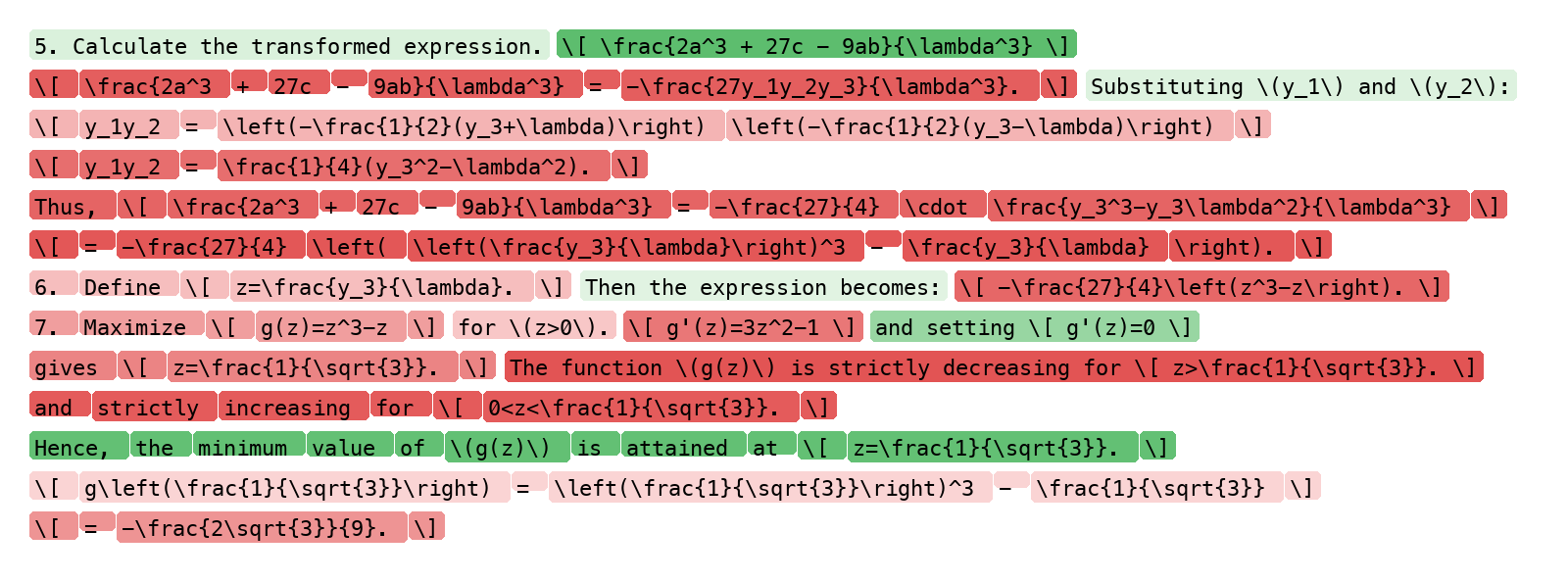}
    \caption{Normalized span-level masking attribution Heatmap }
    \label{fig:heatmap}
\end{figure*}

Reference spans are colored according to their normalized positive Shapley attribution, from green (low) through red (high). Higher attribution indicates that masking the span produces a teacher-distribution change more strongly aligned with the current multi-mask correction.

The system assigns high attribution to the algebraic transformations that reduce the original expression to a single-variable optimization problem, while transition phrases and recoverable intermediate statements receive lower attribution.

\section{Experiment Details}
\label{experimentdetail}
\subsection{Span Segmentation and Masking}
MASS requires partitioning each reference solution $R$ into a sequence of
contiguous spans before applying independent span masking. In our
implementation, we use a rule-based segmentation strategy based on
structural boundaries in the reference text.

Specifically, we identify segmentation boundaries using punctuation marks,
LaTeX delimiters, and formatting symbols that naturally separate reasoning
units. The boundary candidates include sentence-level punctuation
(e.g., periods, question marks, exclamation marks, semicolons, and colons),
mathematical delimiters (e.g., ``\$'', ``\textbackslash(
\textbackslash)'', and ``\textbackslash[
\textbackslash]''),
and LaTeX line-breaking commands. Tokens between two consecutive
boundaries are grouped into one span.

After obtaining the span sequence
$R=[s_1,s_2,\ldots,s_J]$, we perform span-level masking by replacing the
selected spans with a special mask token. Specifically, for a sampled
binary mask vector $m\in\{0,1\}^{J}$, each span $s_j$ is replaced by
``[MASK]'' if $m_j=1$, while the remaining spans are kept unchanged.
The resulting masked reference is then used as the teacher input for
constructing different teacher views.

\subsection{Method Hyperparameters}

We summarize the hyperparameters used in MASS in
Table~\ref{tab:method_hyperparameters}. Here, $K$ is the number of masked
teacher routes, $\rho$ is the expected masking ratio, $\lambda$ controls the
anchor-relative signal, and $\varepsilon$ is the numerical stabilizer in the
selective weighting.

For adaptive masking, $\tau$ regularizes the Shapley regression, $\epsilon$
bounds the masking probabilities away from $0$ and $1$, and $\eta$ controls
the masking-distribution update rate. We maintain the historical attribution
estimate $z$ using an exponential moving average,
\begin{equation}
z^{(n+1)}
=
\gamma_{\mathrm{EMA}} z^{(n)}
+
(1-\gamma_{\mathrm{EMA}})\widehat{\phi}^{(n)},
\end{equation}
where $\gamma_{\mathrm{EMA}}$ controls the smoothing of span attributions
across repeated visits. Unless otherwise specified, we use the same
hyperparameters across all experiments.

\begin{table}[H]
\centering
\caption{Method hyperparameters used in MASS.}
\label{tab:method_hyperparameters}
\small
\setlength{\tabcolsep}{2.5pt}
\renewcommand{\arraystretch}{0.8}
\begin{tabular}{lc}
\toprule
Hyperparameter & Value \\
\midrule
Number of teacher routes ($K$) & 8 \\
Masking ratio ($\rho$) & 0.3 \\
Anchor coefficient ($\lambda$) & 0.5 \\
Weight regularization ($\varepsilon$) & $1\times10^{-6}$ \\
Shapley regression regularization ($\tau$) & 0.1 \\
EMA smoothing ($\gamma_{\mathrm{EMA}}$) & 0.9 \\
Mask probability lower bound ($\epsilon$) & 0.03 \\
Mask adaptation rate ($\eta$) & 0.5 \\
\bottomrule
\end{tabular}
\end{table}

We also provide the training configuration used in our experiments in
Table~\ref{tab:training_configuration}. The evaluation settings are summarized in
Table~\ref{tab:evaluation_configuration}.

\begin{table}[H]
\centering
\caption{Training configuration used in MASS.}
\label{tab:training_configuration}
\begin{tabular}{lc}
\toprule
Training Parameter & Value \\
\midrule
Learning rate & $5\times10^{-6}$ \\
Effective batch size & 32 \\
Maximum sequence length & 1024 tokens \\
Training steps & 250 \\
Number of Generations Per Prompt & 1 \\
Sampling temperature & 1.1 \\
\bottomrule
\end{tabular}
\end{table}

\begin{table}[H]
\centering
\caption{Method hyperparameters used in MASS.}
\label{tab:method_hyperparameters}
\begin{tabular}{lc}
\toprule
Hyperparameter & Value \\
\midrule
Number of teacher routes ($K$) & 8 \\
Masking ratio ($\rho$) & 0.3 \\
Anchor coefficient ($\lambda$) & 0.5 \\
Weight regularization ($\varepsilon$) & $1\times10^{-6}$ \\
Shapley regression regularization ($\tau$) & 0.1 \\
EMA smoothing ($\gamma_{\mathrm{EMA}}$) & 0.9 \\
Mask probability lower bound ($\epsilon$) & 0.03 \\
Mask adaptation rate ($\eta$) & 0.5 \\
\bottomrule
\end{tabular}
\end{table}

\begin{table}[H]
\centering
\caption{Evaluation configuration used in MASS.}
\label{tab:evaluation_configuration}
\begin{tabular}{lc}
\toprule
Evaluation Parameter & Value \\
\midrule
Maximum new tokens & 38912 \\
Thinking mode & Enabled \\
Sampling temperature & 1.0 \\
Top-$p$ & 0.95 \\
Top-$k$ & -1 \\
Min-$p$ & 0.0 \\
Presence penalty & 0.0 \\
Samples per prompt & 12 \\
\bottomrule
\end{tabular}
\end{table}

\end{document}